%% file: main.tex
\documentclass{article}

\usepackage{microtype}
\usepackage{graphicx}
\usepackage{subfigure}
\usepackage{booktabs}

\usepackage{hyperref}
\usepackage{enumerate}

\usepackage[accepted]{icml2023_dp4ml}

\usepackage{amsmath}
\usepackage{amssymb}
\usepackage{mathtools}
\usepackage{amsthm}

\input{symbol.tex}

\usepackage[capitalize,noabbrev]{cleveref}

\theoremstyle{plain}
\newtheorem{theorem}{Theorem}[section]

\newtheorem{lemma}[theorem]{Lemma}

\theoremstyle{definition}

\theoremstyle{remark}

\defcitealias{ChenO23}{Chen \& Orabona (ICML 2023)}

\usepackage[textsize=tiny]{todonotes}

\icmltitlerunning{ Implicit Interpretation of Importance Weight Aware Updates}

\begin{document}

\twocolumn[
\icmltitle{Implicit Interpretation of Importance Weight Aware Updates}

\icmlsetsymbol{equal}{*}

\begin{icmlauthorlist}
\icmlauthor{Keyi Chen}{sch}
\icmlauthor{Francesco Orabona}{sch}

\end{icmlauthorlist}

\icmlaffiliation{sch}{Boston University,  Boston, US}

\icmlcorrespondingauthor{Keyi Chen}{keyichen@bu.edu}
\icmlcorrespondingauthor{Francesco Orabona}{francesco@orabona.com}

\icmlkeywords{Machine Learning, ICML}

\vskip 0.3in
]

\printAffiliationsAndNotice{}  

\begin{abstract}
Due to its speed and simplicity, subgradient descent is one of the most used optimization algorithms in convex machine learning algorithms. However, tuning its learning rate is probably its most severe bottleneck to achieve consistent good performance.
A common way to reduce the dependency on the learning rate is to use implicit/proximal updates. One such variant is the Importance Weight Aware (IWA) updates, which consist of infinitely many infinitesimal updates on each loss function. However, IWA updates' empirical success is not completely explained by their theory.
In this paper, we show for the first time that IWA updates have a strictly better regret upper bound than plain gradient updates in the online learning setting. Our analysis is based on the new framework by \citetalias{ChenO23} to analyze generalized implicit updates using a \textbf{dual formulation}. In particular, our results imply that IWA updates can be considered as approximate implicit/proximal updates.
\end{abstract}

\input{intro.tex}
\input{relat.tex}

\input{def.tex}
\input{mainresult.tex}
\input{iwa}
\input{exp}

\input{conc}

\section*{Acknowledgements}
Francesco Orabona is supported by the National Science Foundation under the grant no. 2046096 ``CAREER: Parameter-free Optimization Algorithms for Machine Learning''.

\bibliography{../../../../learning}
\bibliographystyle{icml2023}

\end{document}

%% file: symbol.tex
\newcommand{\ba}{\boldsymbol{a}}
\newcommand{\bb}{\boldsymbol{b}}

\newcommand{\bq}{\boldsymbol{q}}
\newcommand{\bg}{\boldsymbol{g}}

\newcommand{\bx}{\boldsymbol{x}}
\newcommand{\bu}{\boldsymbol{u}}
\newcommand{\by}{\boldsymbol{y}}

\newcommand{\bz}{\boldsymbol{z}}

\newcommand{\btheta}{\boldsymbol{\theta}}

\newcommand{\argmin}{\mathop{\mathrm{argmin}}}

\newcommand{\interior}{\mathop{\mathrm{int}}}
\newcommand{\dom}{\mathop{\mathrm{dom}}}

\newcommand{\field}[1]{\mathbb{#1}}

\newcommand{\R}{\field{R}}
\newcommand{\Nat}{\field{N}}

\DeclareMathOperator{\Regret}{Regret}


%% file: intro.tex
\section{Introduction}
\label{sec:intro}

In this paper, we are interested in studying variants of gradient updates in the Online Convex Optimization (OCO) setting~\citep{Cesa-BianchiL06, Cesa-BianchiO21, Orabona19}.
In the OCO setting, the learner receives an arbitrary sequence of convex loss functions, selects points before knowing the loss functions, and is evaluated on the values of the loss functions on the points it selects. More in detail, at round $t$ the learner outputs a point $\bx_t$ in a convex feasible set $V\subseteq \R^d$. Then, it receives a loss function $\ell_t: V \to \R$ and it pays the value $\ell_t(\bx_t)$. Given the arbitrary nature of the losses, the learner cannot guarantee to have a small cumulative loss, $\sum_{t=1}^T \ell_t(\bx_t)$. On the other hand, it is possible to minimize the \emph{regret}, that is the difference between the cumulative loss of the algorithm and the one of any arbitrary comparator $\bu \in V$:
\[
\Regret_T(\bu)
\triangleq \sum_{t=1}^T \ell_t(\bx_t) -\sum_{t=1}^T \ell_t(\bu)~.
\]
In particular, a successful OCO algorithm must guarantee a regret that grows sublinearly in time for any $\bu \in V$. In this way, its average performance approaches the one of the best comparator in hindsight.

While in the OCO setting we do not assume anything on how the losses are generated, in the case that the losses are i.i.d. from some fixed (but unknown) distribution we can easily convert a regret guarantee into a convergence rate, using the so-called online-to-batch conversion~\citep{Cesa-bianchiCG02}. Moreover, most of the time the convergence guarantee we obtain in this way are optimal. Hence, the OCO framework allows to analyze many algorithms in the stochastic, batch, and adversarial setting with a single analysis.

The simplest OCO algorithm is Online Gradient Descent (OGD), that, starting from a point $\bx_1 \in V$, in each step updates with
\[
\bx_{t+1} = \Pi_V[\bx_t -\eta \bg_t],
\]
where $\Pi_V$ is the Euclidean projection onto $V$, $\bg_t$ is a gradient of $\ell_t$ in $\bx_t$, and $\eta>0$ is the learning rate. Theoretically and practically, the setting of the learning rate is critical to obtain good performance. In fact, while a learning rate $\eta=O(\tfrac{1}{\sqrt{T}})$ will guarantee an optimal $O(\sqrt{T})$ regret on Lipschitz losses, the constant hidden in the big-O notation can be arbitrarily high.
Moreover, things get even worse when each loss function has an \emph{importance weight} $h_t>0$. This is the case, for example, when each loss function is the loss of the predictor on a classification dataset and we have different classification costs. In this case, the update becomes $\bx_{t+1} = \Pi_V[\bx_t -\eta h_t \bg_t]$ and it should be very intuitive that very large $h_t$ will constraint the learning rate to be small, that in turn will hinder the performance of the algorithm.

In this view, a number of algorithms have been proposed to reduce the sensitivity of OGD to the setting of the learning rate. One of the first successful variants of OGD is the Importance Weight Aware (IWA) updates~\citep{KarampatziakisL11}. The basic idea is to make an infinite number of gradient updates on the loss function $\ell_t$, each of them with an infinitesimal learning rate. This update can be written as the solution of an ODE and it has a closed form for linear predictors and common loss functions. In other words, the IWA updates follow the gradient flow on each loss function.


While the IWA updates are not so known by the machine learning community, they work extremely well in practice. In fact, they are part of the default optimizer used in the large-scale machine learning library Vowpal Wabbit.\footnote{\url{https://vowpalwabbit.org/}}
However, while the IWA updates are very natural and intuitive, the best theoretical guarantee is that their regret upper bound will not be too much worse than the one of plain online gradient descent.

\paragraph{Contributions} In this paper, for the first time we show that the IWA updates have a regret bound that is \emph{better} than the one of plain OGD. We use the very recently proposed framework of Generalized Implicit Follow-the-Regularized-Leader (FTRL)~\citep{ChenO23} that allows to design and analyze more general updates than the classic gradient one. In particular, we show that IWA updates can be seen as approximate implicit/proximal updates because they approximately minimize a certain dual function.

%% file: relat.tex
\paragraph{Related Work}
While IWA updates~\citep{KarampatziakisL11} were motivated by the use of importance weights, they end up being similar to the implicit~\citep{KivinenW97,KulisB10,CampolongoO20} and proximal~\citep{Rockafellar76} updates. In particular, for some losses like the hinge loss, the IWA update and the implicit/proximal update coincide. In this view, they are also similar to the aProx updates~\citep{AsiD19}, which take an implicit/proximal step on truncated linear models. However, as far as we know, no explicit relationship between implicit/proximal updates and IWA updates was known till now. Moreover, the best guarantee for IWA updates just shows that in some restricted cases the regret upper bound is not too much worse than the one of OGD~\citep{KarampatziakisL11}. Finally, all the previous analysis of implicit updates in online learning are conducted in the primal space, while our analysis is done completely in the dual space.

%% file: def.tex
\section{Definitions and Tools}
We define here some basic concepts and tools of convex analysis, we refer the reader to, e.g., \citet{Rockafellar70,BauschkeC11} for a complete introduction to this topic. We will consider extended value function that can assume infinity values too.
A function $f$ is \emph{proper} if it is nowhere $-\infty$ and finite somewhere.
A function $f:V \subseteq \R^d \rightarrow [-\infty, +\infty]$ is \emph{closed} if $\{\bx: f(\bx) \leq \alpha\}$ is closed for every $\alpha \in \R$.
For a proper function $f:\R^d \rightarrow (-\infty, +\infty]$, we define a \emph{subgradient} of $f$ in $\bx \in \R^d$ as a vector $\bg \in \R^d$ that satisfies $f(\by)\geq f(\bx) + \langle \bg, \by-\bx\rangle, \ \forall \by \in \R^d$.
We denote the set of subgradients of $f$ in $\bx$ by $\partial f(\bx)$.
The \emph{indicator function of the set $V$}, $i_V:\R^d\rightarrow (-\infty, +\infty]$, has value $0$ for
$\bx \in V$ and $+\infty$ otherwise.
We denote the \emph{dual norm} of a norm $\|\cdot\|$ by $\|\cdot\|_\star$.
A proper function $f : \R^d \rightarrow (-\infty, +\infty]$ is \emph{$\mu$-strongly convex} over a convex set $V \subseteq \interior \dom f$ w.r.t. $\|\cdot\|$ if $\forall \bx, \by \in V$ and $\forall \bg \in \partial f(\bx)$, we have $f(\by) \geq f(\bx) + \langle \bg , \by - \bx \rangle + \frac{\mu}{2} \| \bx - \by \|^2$.
For a function $f: \R^d\rightarrow [-\infty,\infty]$, we define the \emph{Fenchel conjugate} $f^\star:\R^d \rightarrow [-\infty,\infty]$ as $f^\star(\btheta) = \sup_{\bx \in \R^d} \ \langle \btheta, \bx\rangle - f(\bx)$.
From this definition, we immediately have the Fenchel-Young inequality: $f(\bx) + f^\star(\btheta) \geq \langle \btheta, \bx\rangle, \ \forall \bx, \btheta$.
We will also make use of the following properties of Fenchel conjugates.
\begin{theorem}[{\citep[Theorem~5.7]{Orabona19}}]
\label{thm:props_fenchel}
Let $f:\R^d \rightarrow (-\infty,+\infty]$ be proper. Then, the following conditions are equivalent:
\begin{enumerate}[(a)]
\setlength{\itemsep}{0pt}%
\setlength{\parskip}{0pt}
\vspace{-0.35cm}
\item $\btheta \in \partial f(\bx)$.
\item $\langle \btheta, \by\rangle - f(\by)$ achieves its supremum in $\by$ at $\by=\bx$.
\item $f(\bx)+f^\star(\btheta)=\langle \btheta,\bx\rangle$.
\vspace{-0.35cm}
\end{enumerate}
Moreover, if $f$ is also convex and closed, we have an additional equivalent condition
\begin{enumerate}[(a)]
\setcounter{enumi}{3}
\setlength{\itemsep}{0pt}%
\setlength{\parskip}{0pt}
\vspace{-0.35cm}
\item $\bx \in \partial f^\star(\btheta)$. 
\end{enumerate}
\end{theorem}


%% file: mainresult.tex
\subsection{Generalized Implicit FTRL}
\label{sec:mainres}

In this section, we summarize the generalized formulation of the implicit FTRL algorithm from \citet{ChenO23}.
The main idea is to depart from the implicit or linearized updates,  and directly design updates that improve the upper bound on the regret. 
More in detail, the basic analysis of most of the online learning algorithms is based on the definition of subgradients:
\begin{equation}
\label{eq:subgradient_ineq}
\ell_t(\bx_t) - \ell_t(\bu)
\leq \langle \bg_t, \bx_t-\bu\rangle, \ \forall \bg_t \in \partial \ell_t(\bx_t)~.
\end{equation}
This allows to study the regret on the linearized losses as a proxy for the regret on the losses $\ell_t$.
Instead, \citet{ChenO23} introduce a new fundamental and more general strategy: using the Fenchel-Young inequality, we have
\[
\ell_t(\bx_t) - \ell_t(\bu)  
\leq  \ell_t(\bx_t) - \langle \bz_t,\bu\rangle + \ell_t^\star(\bz_t), \ \forall \bz_t~.
\]
In particular, the algorithm will choose $\bz_t$ from the dual space to make a certain upper bound involving this quantity to be tighter.
This is a better inequality than \eqref{eq:subgradient_ineq} because when we select $\bz_t=\bg_t \in \partial \ell_t(\bx_t)$, using Theorem~\ref{thm:props_fenchel}, we recover \eqref{eq:subgradient_ineq}. So, this inequality subsumes the standard one for subgradients, but, using $\bz_t \in \ell_t(\bx_{t+1})$, it also subsumes the similar inequality used in the implicit case.

The analysis in \citet{ChenO23} shows that the optimal setting of $\bz_t$ is the one that minimizes the sum of two conjugate functions:
\begin{equation}
\label{eq:h}
H_t(\bz)
\triangleq\psi^\star_{t+1,V}(\btheta_{t}-\bz) + \ell^\star_t(\bz),
\end{equation}
where $\psi_{t,V}$ is the restriction of the regularizer used at time $t$ on the feasible set $V$, i.e., $\psi_{t,V}\triangleq\psi_t+i_V$.
However, we can show that any setting of $\bz_t$ that satisfies $H(\bz_t)< H(\bg_t)$ 
guarantees a strict improvement in the worst-case regret w.r.t. using the linearized losses.
The presence of conjugate functions should not be surprising because we are looking for a surrogate gradient $\bz_t$ that lives in the dual space.


\begin{algorithm}[t]
\caption{Generalized Implicit FTRL}
\label{alg:giftrl}
\begin{algorithmic}[1]
{
    \REQUIRE{Non-empty closed set $V\subseteq \R^d$, a sequence of regularizers $\psi_1, \dots, \psi_T : \R^d \rightarrow (-\infty, +\infty]$}
    \STATE{$\btheta_1=\boldsymbol{0}$}
    \FOR{$t=1$ {\bfseries to} $T$}
    \STATE{Output $\bx_t \in \argmin_{\bx \in V} \ \psi_t(\bx) - \langle \btheta_t, \bx\rangle$}
    \STATE{Receive $\ell_t:V \rightarrow \R$ and pay $\ell_t(\bx_t)$}
    \STATE{Set $\bg_t \in \partial \ell_t(\bx_t)$}
    \STATE{Set $\bz_t$ such that $H_t(\bz_t)\leq H_t(\bg_t)$
    where $H_t$
    is defined in \eqref{eq:h}}
    \STATE{Set $\btheta_{t+1}=\btheta_t-\bz_t$}
    \ENDFOR
}
\end{algorithmic}
\end{algorithm}

Once we have the $\bz_t$, we treat them as the gradient of surrogate linear losses.
So, putting it all together, Algorithm~\ref{alg:giftrl} shows the final algorithm.
\citet{ChenO23} prove the following general theorem for it.

%
\begin{theorem}
\label{thm:main}
Let $V\subseteq \R^d$ be closed and non-empty and $\psi_t:V \rightarrow \R$.
With the notation in Algorithm~\ref{alg:giftrl}, define by $F_t(\bx) = \psi_{t}(\bx) + \sum_{i=1}^{t-1} \langle \bz_i, \bx\rangle$, so that $\bx_t \in \argmin_{\bx \in V} \ F_{t}(\bx)$. 
Finally, assume that $\argmin_{\bx \in V} \ F_{t}(\bx)$ and $\partial \ell_t(\bx_t)$ are not empty for all $t$.
For any $\bz_t \in\R^d$ and any $\bu \in \R^d$, we have
\begin{align*}
&\Regret_T(\bu) \leq \psi_{T+1}(\bu) - \min_{\bx \in V} \ \psi_{1}(\bx)\\
&\quad +\sum_{t=1}^T [\psi^\star_{t+1,V}(\btheta_{t}-\bg_t) - \psi^\star_{t,V}(\btheta_t) + \langle \bx_t, \bg_t\rangle-\delta_t] \\
&\quad + F_{T+1}(\bx_{T+1}) - F_{T+1}(\bu),
\end{align*}
where $\delta_t \triangleq H_t(\bg_t)-H_t(\bz_t)$.
\end{theorem}

The Theorem~\ref{thm:main} is stated with very weak assumptions to show its generality, but it is immediate to obtain concrete regret guarantees just assuming, for example, strongly convex regularizers and convex and Lipschitz losses and using well-known methods, such as \citet[Lemma~7.8]{Orabona19}

However, we can already understand why this is an interesting guarantee. Let's first consider the case that $\bz_t=\bg_t$ and the constant regularizer $\frac{1}{2\eta}\|\bx\|^2_2$. In this case, we recover the OGD algorithm. Even the guarantee in the Theorem exactly recovers the best known one~\citep[Corollary 7.9]{Orabona19}, with $\delta_t=0$. 
Instead, if we set $\bz_t$ to be the minimizer of $H_t$, \citet{ChenO23} shows that we recover the implicit/proximal update.
Finally, if we set $\bz_t$ such that $H_t(\bz_t)< H_t(\bg_t)$ we will have that $\delta_t>0$. Hence, in each single term of the sum we have a negative factor that makes the regret bound smaller and we can interpret the resulting update as an \emph{approximate implicit/proximal update}.

%% file: iwa.tex
\section{Importance Weight Aware Updates}

The IWA updates were motivated by the failure of OGD to deal with arbitrarily large importance weights.
In fact, the standard approach to use importance weights in OGD is to simply multiply the gradient by the importance weight. However, when the importance weight is large, we might have an update that is far beyond what is necessary to attain a small loss on it. \citet{KarampatziakisL11} proposed IWA, a computationally efficient way to use importance weights without damaging the convergence properties of OGD.
In particular, IWA updates are motivated by the following invariance property: an example with importance weight $h \in \Nat$ should be treated as if it is an unweighted example appearing $h$ times in the dataset.

More formally, IWA updates are designed for importance weighted convex losses over linear predictors. So, let $\bq_t\in \mathbb{R}^d$ be the $t^\text{th}$ sample and $h_t\in \R_{+}$ its importance weight. Each loss function $\ell_t:\R^d \to \R$ is defined as $\ell_t(\bx)\triangleq \hat{\ell}_t(\langle \bq_t, \bx\rangle)$, where $\bx$ is the predictor, $\langle \bq_t, \bx\rangle$ is the forecast on sample $\bq_t$ of the linear predictor $\bx$, and $\hat{\ell}_t:\R\to\R$ is the $h_t$-weighted convex loss function. For example $\hat{\ell}_t(p) = \frac{h_t}{2} (p-y_t)^2$ for linear regression with square loss,  $\hat{\ell}_t(p)=h_t \ln (1+e^{-y_tp})$, and $\hat{\ell}_t(p)=h_t \max(1-p y_t,0)$ for linear classification with hinge loss.

The key idea of \citet{KarampatziakisL11} is performing a sequence of $N$ updates on each loss function $\ell_t$, each of them with learning rate $\eta/N$, and take $N\to\infty$.
Given the particular shape of the loss functions, all the gradients for a given sample $\bq_t$ points in the same direction: $\nabla \ell_t(\bx) = \hat{\ell}'_t(\langle \bq_t, \bx\rangle) \bq_t$. Therefore, the cumulative effect of performing $N$ consecutive updates in a row on each sample $\bq_t$ amounts to a single update in the direction of $\bq_t$ rescaled by a single scalar. Hence, we just have to find this scalar. More in details, the effect of doing a sequence of infinitesimal updates can be modelled by an ordinary differential equation (ODE), as detailed in the following theorem.
\begin{theorem}[{\citep[Theorem~1]{KarampatziakisL11}}]
\label{th:iwaupdate}
Let $\hat{\ell}$ to be continuously differentiable. Then, the limit for $N\to\infty$ of the OGD update with $N$ updates on the same loss function with learning rate $\eta/N$ is equal to the update
\[
\bx_{t+1} = \bx_t - s_t(1)\bq_t,
\]
where the scaling function $s_t:\R\to \R$ satisfies $s_t(0)=0$ and the differential equation
\[
s'_t(h)
= \eta \hat{\ell}'_t(\langle\bq_t, \bx_t - s_t(h)\bq_t\rangle)~.
\]
\end{theorem}

\paragraph{IWA Updates are Generalized Implicit Updates}
As we said above, IWA updates do not have strong theoretical guarantees. Indeed, we do not even know if they give the same performance of the plain gradient updates in the worst case.
Here, we show that IWA is an instantiation of the generalized implicit FTRL. This implies that its regret upper bound is better than the one of online gradient descent. Moreover, this also gives a way to interpret IWA updates as approximate proximal/implicit updates.

Denote by $p_t \triangleq \langle \bq_t, \bx_t\rangle$, $\bx_t(h) \triangleq \bx_t - s_t(h)\bq_t$,  $p_t(h) \triangleq \langle \bq_t, \bx_t(h)\rangle$, and $\bg_t(h) \triangleq \hat{\ell}'_t(p_t(h)) \bq_t$. Consider the generalized implicit FTRL with regularization $\psi_t(\bx) = \frac{1}{2\eta}\|\bx\|^2_2$ and $V=\R^d$.
Set $\bz_t$ as
\begin{align}
\bz_t 
&\triangleq \int_{0}^1 \! \bg_t(h) \, \mathrm{d}h \nonumber \\
&= \frac{1}{\eta} \left(\int_{0}^1 \! \eta \hat{\ell}_t'(\langle\bq_t, \bx_t-s_t(h)\bq_t\rangle) \, \mathrm{d}h \right) \bq_t \nonumber \\
&=  \frac{1}{\eta} s_t(1) \bq_t~. \label{eq:iwa_z}
\end{align}
Then, the iterates of Algorithm~\ref{alg:giftrl} are the same as the iterates of IWA updates 
\[
\bx_{t+1} 
= \frac{\btheta_t-\bz_t}{1/\eta} 
= \bx_t - \frac{\bz_t}{1/\eta} 
= \bx_t - s_t(1) \bq_t~.
\]
In words, we can now analyze IWA updates as an instantiation of generalized implicit updates.

In particular, Theorem~\ref{th:iwa} shows sufficient conditions on the loss $\hat{\ell}_t$ to guarantee that IWA updates are as good as the subgradient $\bg_t$ by proving that $\bz_t$ satisfy $H_t(\bz_t) \leq H_t(\bg_t)$.
\begin{theorem}
\label{th:iwa}
Assume $s'_t(h)$ to be continuous in $[0,1]$. If $\forall h\in[0,1]$, $\hat{\ell}'_t(p_t(h))$ satisfies one of the following requirements:
\begin{itemize}
\item $\hat{\ell}'_t(p_t(h))\geq0$, $\hat{\ell}'''_t(p_t(h)) \geq 0$
\item $\hat{\ell}'_t(p_t(h))\leq0$, $\hat{\ell}'''_t(p_t(h)) \leq 0$
\end{itemize}
then $\bz_t=\int_{0}^1 \! \bg_t(h) \, \mathrm{d}h$ satisfies
\begin{equation}
\label{eq:iwa}
H_t(\bz_t)\leq H_t(\bg_t)~.
\end{equation}
\end{theorem}

Before proving it, we will need the following technical lemmas.

\begin{lemma}
\label{lemma:ell}
Let $\hat{\ell}_t:\R\to \R$ to be three times differentiable.
\begin{itemize}
\item If $\hat{\ell}'(p_t(h))\geq0, \hat{\ell}'''(p_t(h))\geq 0$, then $s'_t(h)$ is non-negative, non-increasing, convex.
\item If $\hat{\ell}_t'(p_t(h))\leq0, \hat{\ell}_t'''(p_t(h)) \leq 0$, then $s'_t(h)$ is non-positive, non-decreasing, concave.
\end{itemize}
\end{lemma}
\begin{proof}
First, observe that
\begin{align*}
s'_t(h) &= \eta\hat{\ell}'(\langle\bx_t-s_t(h)\bq_t,\bq_t\rangle)=\eta \hat{\ell}'(p_t (h))\\
s''(h)  &= \eta\hat{\ell}''(p_t (h))(-\| \bq_t\|^2)s'_t(h)\\
s'''(h) &= \eta\hat{\ell}'''(p_t (h))\| \bq_t\|^4 (s'_t(h))^2\\
&\quad +\eta \hat{\ell}''(p_t (h))(-\| \bq_t\|^2_2)s''(h)~.
\end{align*}

\textbf{Case 1:} $\hat{\ell}'(p)\geq0, \hat{\ell}'''(p)\geq 0$.
In this case, $s'_t(h)\geq0$, $s''(h)\leq0$,  $s'''(h)\geq0$. That is, $s'_t(h)$ is non-negative, non-increasing, and convex.

\textbf{Case 2:} $\hat{\ell}'(p)\leq0, \hat{\ell}'''(p) \leq 0$.
In this case, $s'_t(h)\leq0$, $s''(h)\geq0$,  $s'''(h)\leq0$. That is, $s'_t(h)$ is non-positive, non-decreasing, and concave.
\end{proof}

\begin{lemma}
\label{lemma:int}
Let $s'_t(h)$ to be continuous in $[0,1]$.
\begin{itemize}
\item If $s'_t(h)$ is convex and non-negative, then $\forall h \in [0,1]$ we have
\[
\frac{1}{2}(s'_t(0)+s'_t(h)) \geq \frac{h}{2}(s'_t(0)+s'_t(h)) \geq s_t(h)~.
\]
\item If $s'_t(h)$ is concave and non-positive, then $\forall h \in [0,1]$ we have
\[
\frac{1}{2}(s'_t(0)+s'_t(h)) \leq \frac{h}{2}(s'_t(0)+s'_t(h)) \leq s_t(h)~.
\]
\end{itemize}
\end{lemma}
\begin{proof}
Given that $s'_t(h)$ is non-negative, we have
$\frac{1}{2}(s'_t(0)+s'_t(h)) \geq \frac{h}{2}(s'_t(0)+s'_t(h))$.
Now, observe that $\frac{h}{2}(s'_t(0)+s'_t(h))$ is the area of the trapezium with first base $s'_t(0)$, second base $s'_t(h)$, and height $h$. Given that the function is convex and non-negative, this area is bigger than the integral of $s'_t$ between 0 and $h$, that is equal to $s_t(h)$, that proves the statement.

We can prove the other case in a similarly way. 
\end{proof}

We can now prove Theorem~\ref{th:iwa}.
\begin{proof}[Proof of Theorem~\ref{th:iwa}]
The left hand side of \eqref{eq:iwa} is equal to
\[
\psi^\star\left(\int_{0}^1 \! \btheta_t- \bg_t(h) \, \mathrm{d}h\right) + \ell_t^\star \left(\int_{0}^1 \! \bg_t(h) \, \mathrm{d}h\right)~.
\]
Since $\psi^\star$ and $\ell_t^\star$ are convex, applying Jensen's inequality, the left hand side of \eqref{eq:iwa} is upper bounded by
\[
\int_{0}^1 \! \psi^\star( \btheta_t- \bg_t(h)) \, \mathrm{d}h + \int_{0}^1 \! \ell_t^\star(\bg_t(h)) \, \mathrm{d}h~.
\]
Moreover, the right hand side of \eqref{eq:iwa} is equal to
\[
\int_{0}^1 \! \psi^\star(\btheta_t- \bg_t) \, \mathrm{d}h + \int_{0}^1 \! \ell_t^\star(\bg_t) \, \mathrm{d}h~.
\]
So, if we can prove that 
\begin{align*}
&\int_{0}^1 \! \psi^\star(\btheta_t- \bg_t(h)) \, \mathrm{d}h + \int_{0}^1 \! \ell_t^\star(\bg_t(h)) \, \mathrm{d}h\\ 
&\quad \leq \int_{0}^1 \! \psi^\star(\btheta_t- \bg_t) \, \mathrm{d}h + \int_{0}^1 \! \ell_t^\star(\bg_t) \, \mathrm{d}h,
\end{align*}
then \eqref{eq:iwa} is proved. 

For this, it is sufficient to prove that $\forall h\in[0,1]$, we have
\[
\psi^\star(\btheta_t- \bg_t(h)) + \ell_t^\star(\bg_t(h))
\leq \psi^\star(\btheta_t- \bg_t) + \ell_t^\star(\bg_t)~.
\]
Given that $\psi_t^\star(\btheta) = \frac{1}{2\lambda}\|\btheta\|^2_2$, where $\lambda = 1/\eta$, and by using the fact that $\langle \bg, \bx \rangle = \ell_t(\bx)+\ell_t^\star(\bg)$, for any pair of $\bx$, $\bg$ satisfying $\bg \in \partial \ell_t(\bx)$, the inequality above can be written as
\begin{align*}
\frac{1}{2\lambda} \| \btheta_t &-\bg_t(h) \|^2_2 + \langle \bg_t(h), \bx_t(h) \rangle - \ell_t(\bx_t(h)) \\
&\quad \leq \frac{1}{2\lambda}\| \btheta_t -\bg_t \|^2_2 + \langle \bg_t,\bx_t \rangle - \ell_t(\bx_t)~.
\end{align*}

Simplifying this inequality, we have
\begin{align*}
&\frac{1}{2\lambda} \| \bg_t(h)\|^2_2 - \frac{1}{2\lambda} \| \bg_t\|^2_2\\
&\quad \leq \ell_t(\bx_t(h)) - \ell_t(\bx_t) -  \langle \bg_t(h), \bx_t(h) -\bx_t \rangle~.
\end{align*}
Since $\ell_t(\bx)$ is convex, we obtain
\begin{align*}
&\ell_t(\bx_t(h)) - \ell_t(\bx_t) -  \langle \bg_t(h), \bx_t(h) -\bx_t \rangle \\
&\quad \geq \langle \bg_t, \bx_t(h) -\bx_t \rangle - \langle \bg_t(h), \bx_t(h) -\bx_t \rangle \\
&\quad = \langle \bg_t(h)-  \bg_t , \bx_t - \bx_t(h) \rangle~.
\end{align*}

So, we just need to prove that 
\[
\frac{1}{2\lambda} \| \bg_t(h)\|^2_2 - \frac{1}{2\lambda} \| \bg_t\|^2_2
\leq \langle \bg_t(h)-  \bg_t , \bx_t - \bx_t(h) \rangle~.
\]
Using $\| \ba\|^2_2 - \|\bb\|^2_2 = \langle \ba-\bb, \ba+\bb\rangle$, the inequality above can be rewritten as
\begin{align*}
\frac{1}{2\lambda} \langle \bg_t(h) -\bg_t, \bg_t(h) + \bg_t \rangle 
&\leq \langle \bg_t(h)-  \bg_t , \bx_t - \bx_t(h) \rangle \\
&= \langle \bg_t(h)-  \bg_t , s_t(h) \bq_t \rangle~.
\end{align*}
That is,
\begin{align*}
\frac{1}{2}& \left( \hat{\ell}'_t(p_t(h)) - \hat{\ell}'_t(p_t)  \right)\left( \frac{1}{\lambda} \hat{\ell}'_t(p_t(h)) + \frac{1}{\lambda} \hat{\ell}'_t(p_t)  \right) \| \bq_t\|^2_2 \\
&\leq \left( \hat{\ell}'_t(p_t(h)) - \hat{\ell}'_t(p_t)  \right) s_t(h) \| \bq_t\|^2_2~.
\end{align*}
Since $s'_t(h) = \frac{1}{\lambda} \hat{\ell}'_t(p_t(h))$, multiplying both side by $1/\lambda$, the above inequality becomes
\begin{align}
&\frac{1}{2}(s'_t(h) - s'_t(0))(s'_t(0)+s'_t(h)) \nonumber \\
&\quad \leq (s'_t(h) - s'_t(0))s_t(h)~. \label{eq:iwa_s}
\end{align}

Now, we consider two cases.

\textbf{Case 1:} $\hat{\ell}'_t(p_t(h))\geq0$ and $\hat{\ell}_t'''(p_t(h))\geq 0$.
By Lemma~\ref{lemma:ell}, $s'_t(h)$ is non-negative, non-increasing, and convex. So, in particular we have $s'_t(h) - s'_t(0)\leq 0$.
In this case, by Lemma~\ref{lemma:int}, we have $\frac{1}{2}(s'_t(0)+s'_t(h)) \geq s_t(h)$.

\textbf{Case 2:} $\hat{\ell}'_t(p_t(h))\leq0$, and $\hat{\ell}_t'''(p_t(h)) \leq 0$.
By Lemma~\ref{lemma:ell}, $s'_t(h)$ is non-positive, non-decreasing, concave. So, in particular, we have $s'_t(h) -s'_t(0) \geq 0$.
In this case,  by Lemma~\ref{lemma:int}, we have $\frac{1}{2}(s'_t(0)+s'_t(h)) \leq s_t(h)$.

Combining the two cases, we conclude that \eqref{eq:iwa_s} is true that implies that \eqref{eq:iwa} is true as well.
\end{proof}

\subsection{Examples of Losses and IWA Updates}

Now, we present some examples of loss functions that satisfy the requirements of Theorem~\ref{th:iwa} and their corresponding IWA updates from \citet{KarampatziakisL11}. In all the following examples, the prediction of the algorithm on sample $\bq$ is $\langle \bq, \bx\rangle$.

\textbf{Logistic loss: $\hat{\ell}(p) = h \ln(1+e^{-y p})$}.

The IWA update is $\frac{W(e^{h \eta \|\bq\|^2_2 + y p +e^{yp}})-h \eta \|\bq\|^2_2-e^{y p}}{y\|\bq\|^2_2}$ for $y \in \{-1, 1\}$, where $W(x)$ is the Lambert function.
We have that
\[
\hat{\ell}' (p) = \frac{-y h}{1+e^{py}}, \quad 
\hat{\ell}''' (p) = h y^3(-e^{-py-1})~.
\]
When $y \geq 0$, $\hat{\ell}'(p) \leq 0$ and $\hat{\ell}''' (p)\leq 0$. When $y\leq 0$, $\hat{\ell}'(p) \geq 0$ and $\hat{\ell}''' (p)\geq 0$. 

\textbf{Exponential loss: $\hat{\ell}(p)=e^{-yp}$}.

The IWA update is $\frac{p y -\ln(h \eta \|\bq\|_2^2+e^{p y})}{\|\bq\|_2^2 y}$ for $y \in \{-1, 1\}$.
We have that
\[
\hat{\ell}' (p) = y(-e^{-py}), \quad 
\hat{\ell}''' (p) = y^3(-e^{-py})~.
\]
When $y \geq 0$, $\hat{\ell}l'(p) \leq 0$ and $\hat{\ell}''' (p)\leq 0$.  
When $y\leq 0$, $\hat{\ell}'(p) \geq 0$ and $\hat{\ell}''' (p)\geq 0$.

\textbf{Logarithmic loss: $\hat{\ell}(p) = y \ln(y/p) + (1-y) \ln((1-y)/(1-p))$}.

The IWA update is $\frac{p -1 + \sqrt{(p-1)^2+2h \eta \|\bq\|^2_2}}{\|\bq\|_2^2}$ for $y=0$, and $\frac{p - \sqrt{p^2+2h \eta \|\bq\|^2_2}}{\|\bq\|_2^2}$ for $y=1$.
\begin{itemize}
\item if y=0
\[
\hat{\ell}' (p)=\frac{1}{1-p},\quad
\hat{\ell}''' (p)= -\frac{2}{(p-1)^3}~.
\]
\item if y=1
\[
\hat{\ell}' (p)=-\frac{1}{p}, \quad
\hat{\ell}''' (p)= -\frac{2}{p^3}~.
\]
\end{itemize}
For both cases, $\hat{\ell}'(p)$ and $\hat{\ell}'''(p)$ will have the same sign.
 
%

\textbf{Squared loss: $\hat{\ell}(p) = \frac{1}{2}(y-p)^2$.}
The IWA update is $\frac{p-y}{\|\bq\|^2_2}(1-e^{-h \eta \|\bq\|^2_2})$.
\[
\hat{\ell}' (p)=p-y,\quad \hat{\ell}''' (p)=0~.
\]
In this case, the sign of the first derivative can change. However, according to Section 4.2 of \citet{KarampatziakisL11}, for the squared loss IWA will not overshoot the minimum. This means that for any $h \in [0,1]$, $p_t(h)-y$ will always have the same sign, so the conditions are verified.

%% file: exp.tex
\section{Empirical Evaluation}
\label{sec:exp}

\begin{figure}[ht]
\centering
\includegraphics[width=0.4\textwidth]{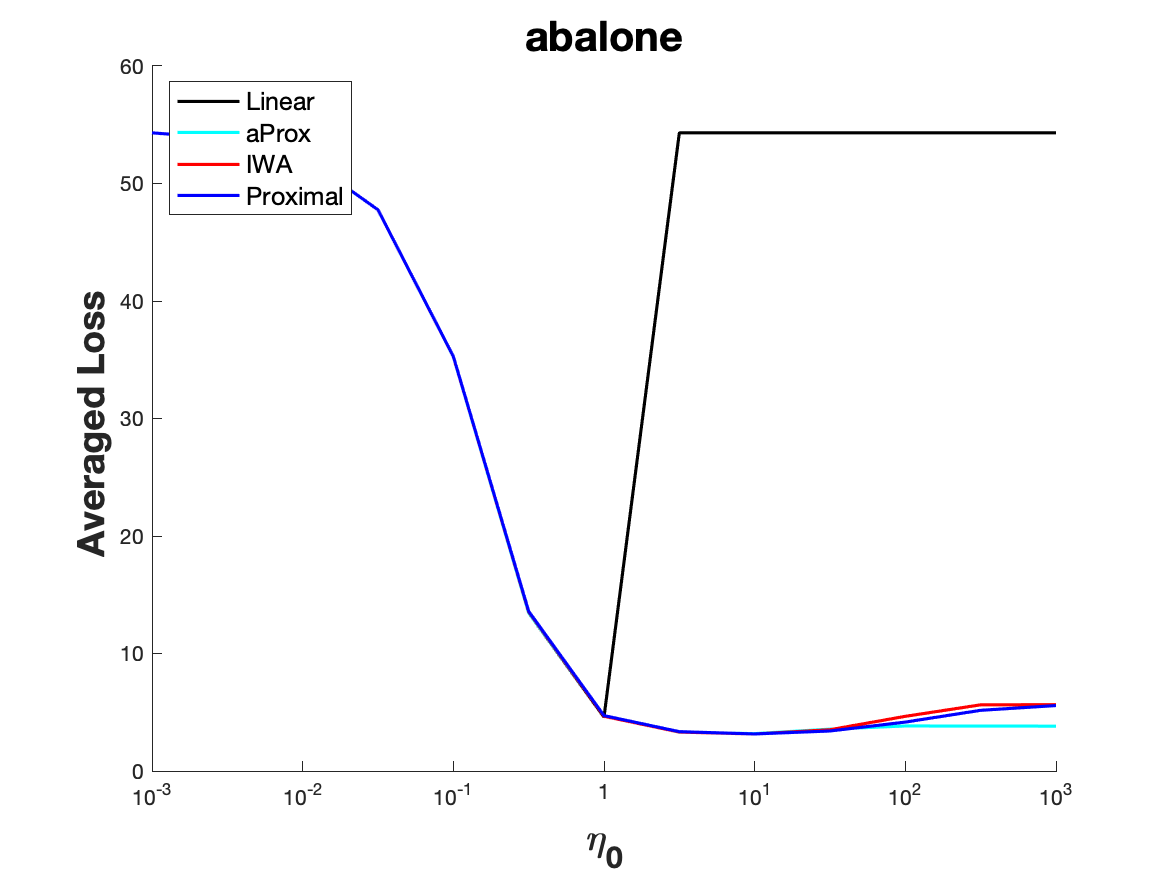}  
\includegraphics[width=0.4\textwidth]{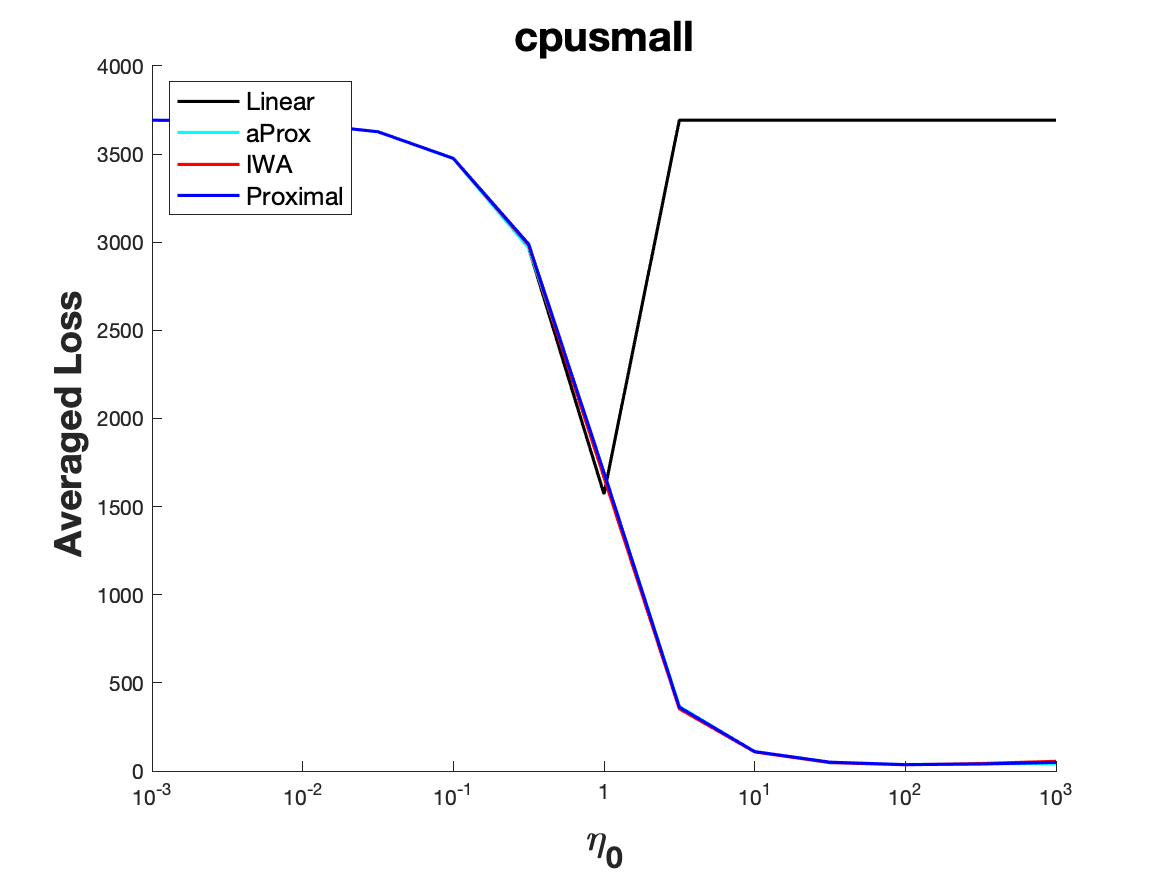} 
\includegraphics[width=0.4\textwidth]{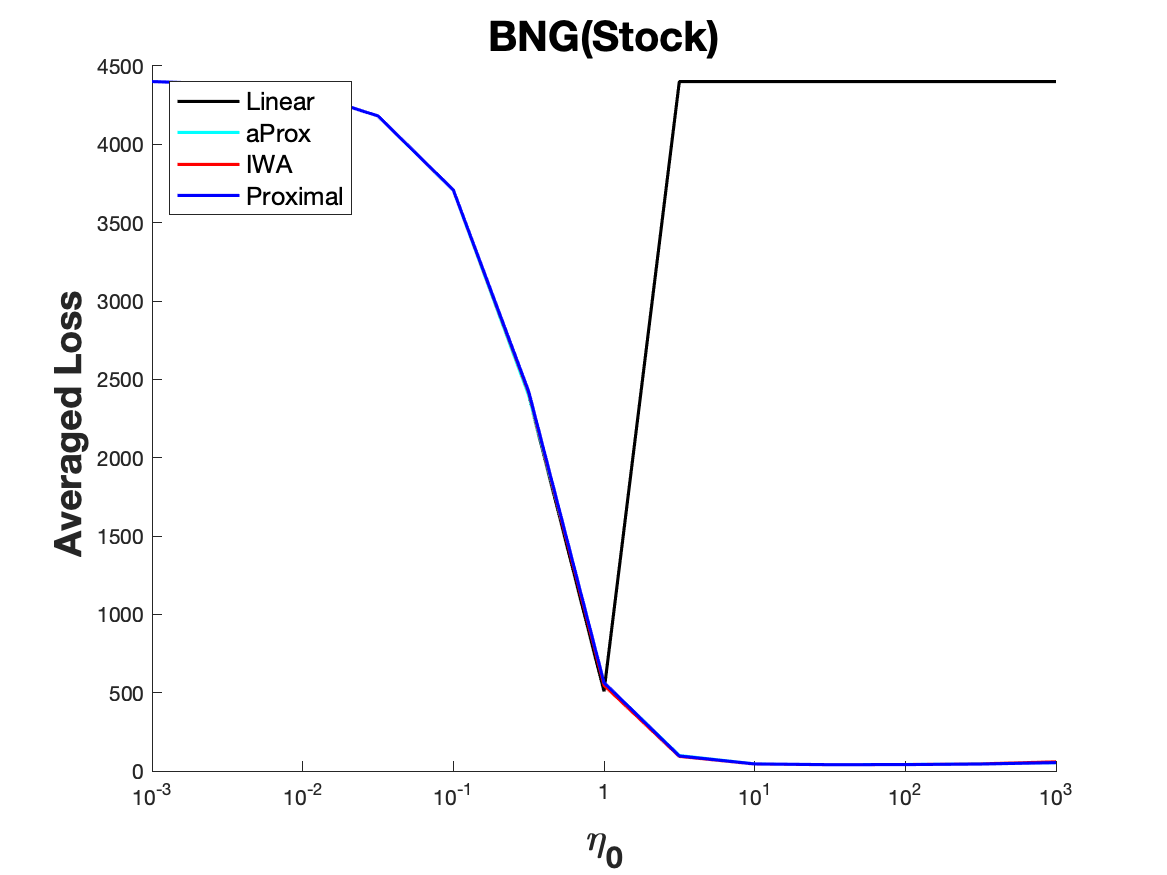} 
\vspace{-0.5cm}
\caption{Squared loss, averaged loss vs. hyperparameter $\eta_0$.}
\label{fig:reg_plot}
\end{figure}

\begin{figure}[t]
\centering
\includegraphics[width=0.4\textwidth]{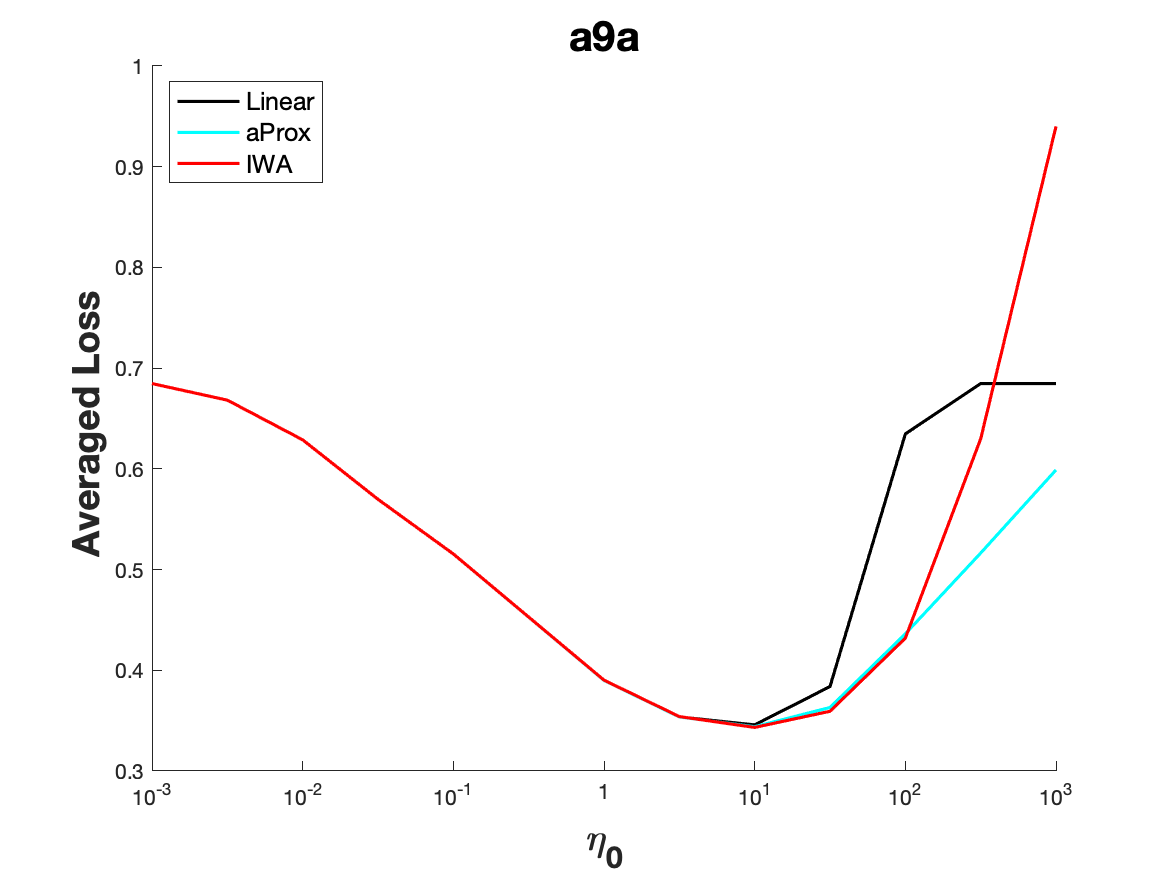}  
\includegraphics[width=0.4\textwidth]{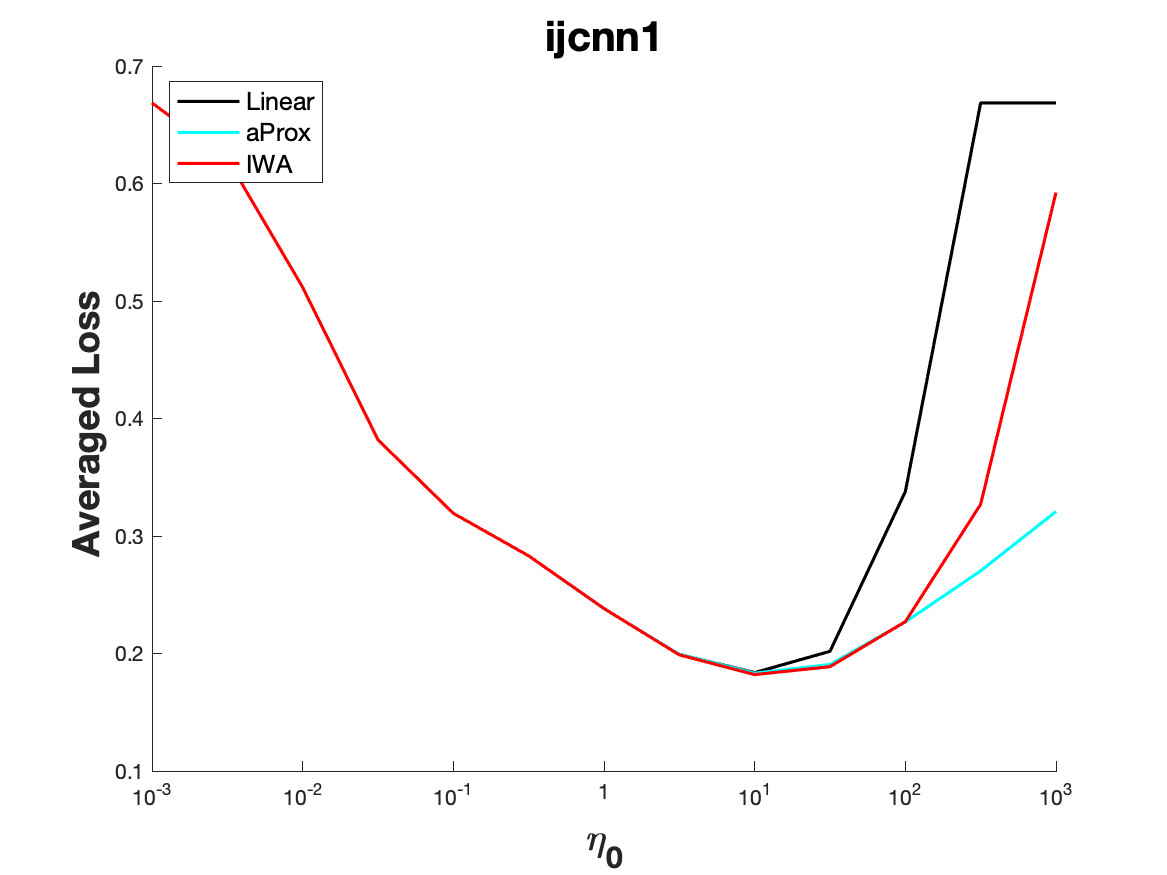} 
\includegraphics[width=0.4\textwidth]{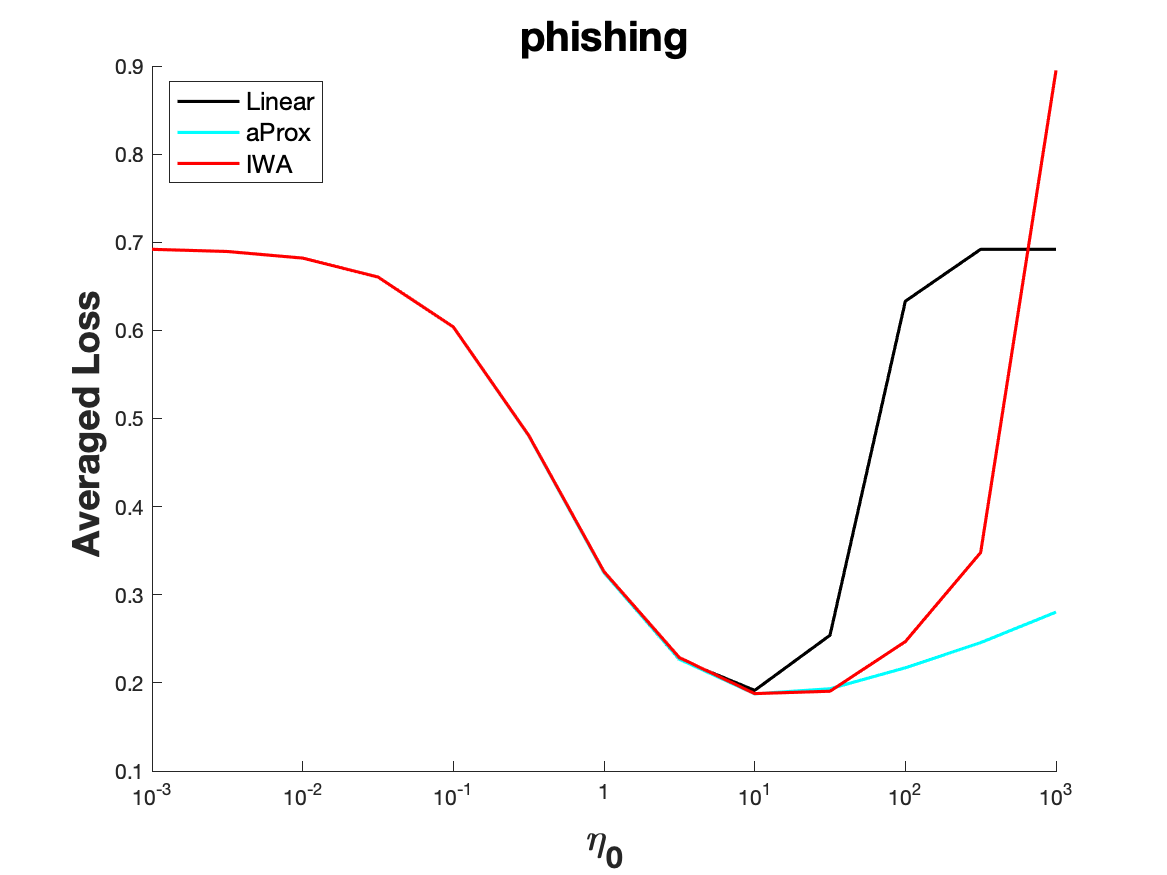} 
\vspace{-0.5cm}
\caption{Logistic loss, averaged loss vs. hyperparameter $\eta_0$.}
\label{fig:logistic_plot}
\end{figure}

IWA as an instantiation of generalized implicit updates, guarantees that in the worst case scenario, it will be at least as good as FTRL with linear models.  Generalized implicit FTRL is a flexible framework in that $\bz_t$ has a bunch of choices.  
In this section, we compare the performance of IWA updates with different choices of $\bz_t$ in Algorithm~\ref{alg:giftrl}, when $\psi_t=\frac{\lambda_t}{2}\|\bx\|_2^2$ and $\lambda_t$ is set as in \citet[Corollary 4.3]{ChenO23}.
In particular, we consider:
\begin{itemize}
\setlength{\itemsep}{0pt}%
\setlength{\parskip}{0pt}
\vspace{-0.3cm}
\item FTRL with linearized losses: $\bz_t=\bg_t$;
\item Implicit FTRL with aProx updates: $\bz_t = \min\left\{ 1,\frac{\lambda_t \ell_t(\bx_t)}{\| \bg_t\|^2}\right\}\bg_t$;
\item Implicit FTRL with IWA updates: $\bz_t = \frac{1}{\eta}s_t(1) \bq_t$;
\item Implicit FTRL with proximal updates.
\vspace{-0.3cm}
\end{itemize}

We conduct linear prediction experiments on datasets from LibSVM~\citep{ChangL11}. We show experiments on classification tasks using the logistic loss, and regression tasks with squared loss. We normalize the datasets and added a constant bias term to the features. Given that in the online learning setting, we do not have the training data and validation data to tune the initial learning rate, we will plot the averaged loss, $\frac{1}{t}\sum_{i=1}^t \ell_i(\bx_i)$, versus different choice of initial learning rate $\eta_0$, that at the same time show the algorithms' sensitivity to the hyperparameter $\eta_0$ and their best achievable performance. We consider $\eta_0 \in [10^{-3},10^3]$. Each algorithm is run 10 times with different shuffling of the data and we plot the average of the averaged losses.

Figure~\ref{fig:reg_plot} shows the averaged loss versus a different selection of hyperparameter $\eta_0$ for regression tasks with squared loss. The figure demonstrates that FTRL with linearized updates is very sensitive to the choice of the hyperparameter $\eta_0$, while the implicit FTRL updates are robust to different setting of hyperparameters. The range of learning rate selection is much broader than that of FTRL with linear models.   

Figure~\ref{fig:logistic_plot} shows the averaged loss versus different selections of hyperparameter $\eta_0$ for classification tasks with logistic loss.  In the experiments, implicit FTRL with IWA updates improves upon FTRL with linearized models. Both IWA updates and aProx updates allow broader learning rate selection. This is in line with previous results in \citet{AsiD19} in the offline setting. Note that in this case the proximal operator does not have a closed-form solution. Yet, IWA provided a way to approximate proximal updates efficiently and, in some sense, to enjoy the stability of proximal updates.

%% file: conc.tex
\section{Conclusion}
\label{sec:conc}
Generalized implicit FTRL as a flexible framework, allows the design of new algorithms and immediate theoretical analysis. By proving that IWA is a concrete instantiation of this new framework, we prove that IWA updates have a regret bound that is \emph{better} than the one of plain OGD, and provide a perspective to interpret IWA updates as approximate implicit/proximal updates.